\documentclass[conference]{IEEEtran}
\IEEEoverridecommandlockouts
\usepackage{cite}
\usepackage{amsmath,amssymb,amsfonts}
\usepackage{mathrsfs}
\usepackage{graphicx}
\usepackage{textcomp}
\usepackage{xcolor}
\usepackage{subfigure}
\usepackage[vlined,boxed,ruled,linesnumbered]{algorithm2e}
\SetKwProg{Fn}{Function}{}{}
\usepackage{algpseudocode}
\def\BibTeX{{\rm B\kern-.05em{\sc i\kern-.025em b}\kern-.08em
    T\kern-.1667em\lower.7ex\hbox{E}\kern-.125emX}}

\newtheorem{definition}{Definition}
\newtheorem{theorem}{Theorem}

\newtheorem{remark}{Remark}

\begin{document}

\title{Guaranteed Quantization Error Computation for Neural Network Model Compression\\
\thanks{This research was supported by the National Science Foundation, under NSF CAREER Award 2143351, and NSF CNS Award no. 2223035.}
}

\author{\IEEEauthorblockN{Wesley Cooke}
\IEEEauthorblockA{\textit{School of Computer and Cyber Sciences} \\
\textit{Augusta University}\\
Augusta, USA \\
wcooke@augusta.edu}
\and
\IEEEauthorblockN{Zihao Mo}
\IEEEauthorblockA{\textit{School of Computer and Cyber Sciences} \\
\textit{Augusta University}\\
Augusta, USA \\
zmo@augusta.edu}
\and
\IEEEauthorblockN{Weiming Xiang}
\IEEEauthorblockA{\textit{School of Computer and Cyber Sciences} \\
\textit{Augusta University}\\
Augusta, USA \\
wxiang@augusta.edu}
}

\maketitle

\begin{abstract}
Neural network model compression techniques can address the computation issue of deep neural networks on embedded devices in industrial systems. The guaranteed output error computation problem for neural network compression with quantization is addressed in this paper. A merged neural network is built from a feedforward neural network and its quantized version to produce the exact output difference between two neural networks. Then, optimization-based methods and reachability analysis methods are applied to the merged neural network to compute the guaranteed quantization error. Finally, a numerical example is proposed to validate the applicability and effectiveness of the proposed approach. 
\end{abstract}


\section{Introduction}

 Neural networks have been demonstrated to be powerful and effective tools to solve complex problems such as image processing \cite{SCHMIDHUBER201585}, high-performance adaptive control \cite{HUNT19921083}, etc. Due to the increasing complexity of the problems in various applications, the scale and complexity of neural networks also grow exponentially to meet the desired accuracy and performance.  Recent progress of machine learning, such as training and using a new generation of large neural networks, heavily depends on the availability of exceptionally large computational resources, e.g., the Transformer model with neural architecture search proposed in \cite{vaswani2017attention}, if trained from scratch for each case,  requires 274,120 hours training on 8 NVIDIA P100 GPUs \cite{strubell2019energy}. Additionally, even for already trained neural networks, the verification process is also quite time- and resource-consuming, e.g., some simple properties in the ACAS Xu neural network of a 5-layer simple structure proposed in \cite{katz2017reluplex} need more than 100 hours to be verified. To avoid unaffordable computation when using neural networks, a variety of neural network acceleration and compression methods are proposed such as neural network pruning and quantization, which can significantly reduce the size and memory footprint of neural networks as well as expedite the speed of model inference.
 
Quantization as a reduction method is mainly concerned with the amount of memory utilized for the learnable parameters of a neural network. The data type for the weights and biases of a typical neural network is usually expressed as 32-bit floating-point values that will carry out millions of floating-point operations during inference time. Quantization aims to shrink the memory footprint of deep neural networks by reducing the number of bits used to store the values for the learnable parameters and activations. This is not only ideal for application scenarios where memory resources may be restricted such as embedded systems or microcontroller environments, but with the selected weight representation you could potentially facilitate faster inference using cheaper arithmetic operations \cite{RENNFES}. With the reduction in parameter bit precision, however, it is typical that a quantized neural network will perform worse in terms of accuracy than its non-quantized counterpart using gradient-based learning methods. However, these drops in accuracy are usually considered minimal and worth the given benefit in memory reduction and inference speed-up.
There exists much literature describing various techniques of quantization and successful results thereof, including works utilizing stochastic rounding to select weight values beneficial to gradient training \cite{HOHFELD1992291} and applications on modern deep architectures \cite{gupta2015deep}, as well as quantization methods that reduce the number of multiplication operations required during training time \cite{NNWFM}. Significant research has also been done on quantization-aware training methods \cite{Park_2018_ECCV}, where the loss in accuracy due to bit precision reduction is minimized. Some of these quantization-aware training methods utilize a straight-through gradient estimator (STE) \cite{EPGTSNCC} to more appropriately select weights during network training that minimizes accuracy and further reduces computational burden \cite{TDNNBWDP}.

As quantization methods are used for neural network reduction, there inevitably exist discrepancies between the performances of original and compressed neural networks. In this work, we propose a computationally tractable approach to compute the guaranteed output error caused by quantization. A merged neural network is constructed to generate the output differences between two neural networks, and then reachability analysis on the merged neural network can be performed to obtain the guaranteed error. 
The remainder of the paper is organized as follows: Preliminaries are given in Section II. The main results on quantization error computation are presented in Section III. A numerical example is given in Section IV. The conclusion is presented in Section V.


\section{Preliminaries}

In this work, we consider a class of fully-connected feedforward neural networks which can be described by the following recursive equations
\begin{align}\label{eq:nn}
\begin{cases}
    \mathbf{u}_{0} = \mathbf{u}
    \\
    \mathbf{u}_{\ell} = \phi_{\ell}(\mathbf{W}_{\ell} \mathbf{u}_{\ell-1}+\mathbf{b}_\ell),~\ell = 1,\ldots,L
    \\
    \mathbf{y} = \mathbf{u}_{L}
    \end{cases}
\end{align}
where $\mathbf{u}_{0} = \mathbf{u} \in \mathbb{R}^{n_u}$ is the input vector of the neural network, $\mathbf{y} = \mathbf{u}_{L}  \in \mathbb{R}^{n_y}$ is the output vector of the neural network, $\mathbf{W}_\ell \in \mathbb{R}^{n_{\ell}\times n_{\ell-1}}$ and $\mathbf{b}_{\ell} \in \mathbb{R}^{n_{\ell}}$ are weight matrices and bias vectors for the $\ell$-th layer, respectively. $\phi_\ell = [\psi_{\ell},\cdots,\psi_{\ell}]$ is the concatenation of activation functions of the $\ell$-th layer in which $\psi_{\ell}:\mathbb{R} \to \mathbb{R}$ is the activation function, e.g., such as logistic, tanh, ReLU, sigmoid functions. In addition, the input-output mapping of the above neural network $\Phi:\mathbb{R}^{n_u} \to \mathbb{R}^{n_y}$ is denoted in the form of
\begin{align} \label{eq:compact_nn}
    \mathbf{y} = \Phi(\mathbf{u})
\end{align}
where $\mathbf{u} \in \mathbb{R}^{n_u}$ and $\mathbf{y} \in \mathbb{R}^{n_y}$ are the input and output of the neural network, respectively. 

A common quantization procedure $\mathsf{Q}(\cdot)$ to map a floating point value $r$ to an integer can be formulated as what follows
\begin{align}\label{eq:quantize}
    \mathsf{Q}(r) =  \mathsf{int}(r/S)-Z
\end{align}
where $S$  is a floating point value as a scaling factor, and $Z$ is an integer value that represents $0$ in the quantization policy which could be $0$ or other values. The $\mathsf{int}:\mathbb{R} \to \mathbb{Z}$ is the function rounding the floating point value of an integer. 

To reduce the size and complexity of the neural network, the quantization procedure is implemented on the neural network parameters, i.e., weights and biases. The quantized version of the neural network  (\ref{eq:nn}) is in the form of
\begin{align}
    \begin{cases}
    \mathbf{u}_{0} = \mathbf{u}
    \\
    \mathbf{u}_{\ell} = \phi_{\ell}(\mathsf{Q}(\mathbf{W}_{\ell}) \mathbf{u}_{\ell-1}+\mathsf{Q}(\mathbf{b}_\ell)),~\ell = 1,\ldots,L
    \\
    \mathbf{y} = \mathbf{u}_{L}
    \end{cases}
\end{align}
where $\mathsf{Q}(\mathbf{W}_\ell) \in \mathbb{Z}^{n_{\ell}\times n_{\ell-1}}$ and $\mathsf{Q}(\mathbf{b}_{\ell}) \in \mathbb{Z}^{n_{\ell}}$ are the quantized weight matrices and bias vectors for the $\ell$-th layer under the quantization process $\mathsf{Q}(\cdot)$. Furthermore, the quantized version of neural network $\Phi:\mathbb{R}^{n_u} \to \mathbb{R}^{n_y}$ is expressed as $\Phi_{\mathsf{Q}}:\mathbb{Z}^{n_u} \to \mathbb{Z}^{n_y}$ in the form of 
\begin{align} \label{eq:q_compact_nn}
    \mathbf{y} = \Phi_{\mathsf{Q}}(\mathbf{u}) .
\end{align}

The quantization can significantly reduce the size and computational complexity of a neural network, e.g., mapping the 32-bit floating point representation to an 8-bit integer representation, leading to smaller
models that can fit in hardware with high computational efficiency. However, the price to pay is the loss of performance and precision post-quantization. To formally characterize the performance loss caused by quantization, a reasonable expectation is to compute the quantization error between the neural network and its quantized version. 

\begin{definition}\label{def:error}
Given a tuple $\mathbb{M} \triangleq  \langle \Phi, \mathsf{Q}, \mathcal{U} \rangle$ where $\Phi$ is a neural network defined by (\ref{eq:nn}), $\mathsf{Q}$ is the quantization process of (\ref{eq:quantize}) producing quantized neural network $\Phi_{\mathsf{Q}}$, and $\mathcal{U} \in \mathbb{R}^{n_u}$ is a compact input set, the guaranteed quantization error is defined by 
\begin{align}
    \rho(\mathbb{M}) = \sup_{\mathbf{u}\in \mathcal{U}}\left\|\Phi(\mathbf{u})  -\Phi_{\mathsf{Q}}(\mathbf{u}) \right\|
\end{align}
where $\Phi_{\mathsf{Q}}$ is the quantized neural network of $\Phi$. 
\end{definition}

\begin{remark}
The assumption that the input set $\mathcal{U}$ is a compact set is reasonable since neural networks are rarely applied to raw data sets. Instead, standardization and rescaling techniques such as normalization are used which ensure the inputs are always within a compact set such as $[0,1]$ or $[-1,1]$.  Given the compact input set $\mathcal{U}$ which contains all possible input to the neural network, the guaranteed quantization error $\rho(\mathbb{M}) $ characterizes the upper bound for the difference between the outputs of neural network $\Phi$ and its quantized version $\Phi_{\mathsf{Q}}$ generated from the same inputs in set $\mathcal{U}$, which quantifies the discrepancy caused by the quantization process $\mathsf{Q}$ in terms of outputs. 
\end{remark}

\section{Quantization Error Computation}

To address the quantization error computation problem, the key is to estimate a $\gamma >0 $ such that $\rho(\mathbb{M}) \le \gamma$. Due to the complexity of the neural network, it is challenging to estimate the $\gamma$ directly from the discrepancy of $\Phi(\mathbf{u})  -\Phi_{\mathsf{Q}}(\mathbf{u})$. Other than directly analyzing the discrepancy of two neural networks, we proposed to construct a new fully-connected neural network $\tilde{\Phi}$ merged from $\Phi$ and $\Phi_{\mathsf{Q}}$ which is expected to produce the discrepancy of the outputs of two neural networks, i.e., $\tilde{\Phi}(\mathbf{u})=\Phi(\mathbf{u})  -\Phi_{\mathsf{Q}}(\mathbf{u})$,  and then search for the upper bound of the outputs of the merged neural network.

Given $L$-layer neural network $\Phi$ and quantized $\Phi_{\mathsf{Q}}$, the merged neural network $\tilde{\Phi}:\mathbb{R}^{n_u} \to \mathbb{R}^{n_y}$ is constructed with $L+1$ layers as follows:
\begin{align}\label{eq:m_nn}
\begin{cases}
    \tilde{\mathbf{u}}_{0} = \mathbf{u}
    \\
    \tilde{\mathbf{u}}_{\ell} = \tilde{\phi}_{\ell}(\tilde{\mathbf{W}}_{\ell} \tilde{\mathbf{u}}_{\ell-1}+\tilde{\mathbf{b}}_\ell),~\ell = 1,\ldots,L+1
    \\
    \tilde{\mathbf{y}} = \tilde{\mathbf{u}}_{L+1}
    \end{cases}
\end{align}
where
\begin{align} \label{eq:thm1_1}
\tilde{\mathbf{W}}_{\ell} & = \begin{cases}
    \begin{bmatrix}
    \mathbf{W}_{1} \\ \mathsf{Q}(\mathbf{W}_{1})
    \end{bmatrix}, &\ell = 1
    \\
    \begin{bmatrix}
    \mathbf{W}_{\ell} & \mathbf{0}_{{n_{\ell}} \times { n_{\ell-1}}}
    \\
    \mathbf{0}_{{ n_{\ell-1}} \times {n_{\ell}}} & \mathsf{Q}({\mathbf{W}}_{\ell})
    \end{bmatrix}, &1 < \ell \le L
    \\
    \begin{bmatrix}
    \mathbf{I}_{n_y} & -\mathbf{I}_{n_{y}} 
    \end{bmatrix}, & \ell =L+1
    \end{cases}
\\
     \label{eq:thm1_2}
    \tilde{\mathbf{b}}_{\ell} &= \begin{cases}
    \begin{bmatrix}
    \mathbf{b}_\ell \\ 
    \mathsf{Q}(\mathbf{b}_\ell)
    \end{bmatrix}, & 1 \le \ell \le  L
    \\
    \begin{bmatrix}
    \mathbf{0}_{{2n_y} \times 1}     \end{bmatrix}, & \ell = L+1
    \end{cases} 
    \\
    \label{eq:thm1_3}
    \tilde \phi_\ell(\cdot) &= \begin{cases}
      \phi_\ell(\cdot) 
,& 1 \le \ell \le  L
    \\
    \mathsf{L}(\cdot), & \ell = L+1
    \end{cases} 
    \end{align}
where $\mathsf{L}(\cdot)$ is linear transfer function, i.e., $x = \mathsf{L}(x)$.

\begin{theorem}\label{thm1}
Given a tuple $\mathbb{M} \triangleq  \langle \Phi, \mathsf{Q}, \mathcal{U} \rangle$ where $\Phi$ is a neural network defined by (\ref{eq:nn}), $\mathsf{Q}$ is the quantization process of (\ref{eq:quantize}), and $\mathcal{U} \in \mathbb{R}^{n_u}$ is a compact input set,  the guaranteed quantization error $\rho(\mathbb{M})$ can be computed by
\begin{align}
    \rho(\mathbb{M}) = \sup_{\mathbf{u}\in \mathcal{U}}\left\|\tilde\Phi(\mathbf{u})\right\|
\end{align}
where $\tilde\Phi$ is a fully-connected neural network defined in (\ref{eq:m_nn}). 
\end{theorem}
\begin{proof} First, let us consider $\ell = 1$. Given an input $\tilde{\mathbf{u}}_0 =\mathbf{u}  \in \mathbb{R}^{n_u}$, one can obtain that
\begin{align} \label{eq:p_thm1_1}
  \tilde{\mathbf{u}}_1 = \tilde{\phi}_{1}(\tilde{\mathbf{W}}_{1} \tilde{\mathbf{u}}_0+\tilde{\mathbf{b}}_1) = \begin{bmatrix}
   \phi_{1}(\mathbf{W}_{1} \tilde{\mathbf{u}}_0+\mathbf{b}_1) \\ \phi_{1}(\mathsf{Q}(\mathbf{W}_{1}) \tilde{\mathbf{u}}_{0}+\mathsf{Q}(\mathbf{b}_1))
    \end{bmatrix}.
\end{align}

Then, we consider $1 < \ell \le L$. Starting from $\ell =2$, we have 
\begin{align*}
  \tilde{\mathbf{W}}_{2} \tilde{\mathbf{u}}_{1}
     & = \begin{bmatrix}
    \mathbf{W}_{2} & \mathbf{0}_{{n_{2}} \times { n_{1}}}
    \\
    \mathbf{0}_{{ n_{2}} \times {n_{1}}} & \mathsf{Q}({\mathbf{W}}_{2}) 
    \end{bmatrix} \begin{bmatrix}
   \phi_{1}(\mathbf{W}_{1} \tilde{\mathbf{u}}_0+\mathbf{b}_1) \\ \phi_{1}(\mathsf{Q}(\mathbf{W}_{1}) \tilde{\mathbf{u}}_{0}+\mathsf{Q}(\mathbf{b}_1))
    \end{bmatrix}
    \\
   &= \begin{bmatrix}
   \mathbf{W}_{2}\phi_{1}(\mathbf{W}_{1} \tilde{\mathbf{u}}_0+\mathbf{b}_1) \\ \mathsf{Q}({\mathbf{W}}_{2}) \phi_{1}(\mathsf{Q}(\mathbf{W}_{1}) \tilde{\mathbf{u}}_{0}+\mathsf{Q}(\mathbf{b}_1))
    \end{bmatrix}
 \end{align*} .
Furthermore, it leads to
\begin{align*}
    \tilde{\mathbf{u}}_2 &= \tilde{\phi}_{2}(\tilde{\mathbf{W}}_{2} \tilde{\mathbf{u}}_1+\tilde{\mathbf{b}}_2) 
    \\
    &=  \begin{bmatrix}
   \phi_2(\mathbf{W}_{2}\phi_{1}(\mathbf{W}_{1} \tilde{\mathbf{u}}_0+\mathbf{b}_1)+\mathbf{b}_2)\\ \phi_2(\mathsf{Q}({\mathbf{W}}_{2}) \phi_{1}(\mathsf{Q}(\mathbf{W}_{1}) \tilde{{\mathbf{u}}}_{0}+\mathsf{Q}(\mathbf{b}_1))+\mathbf{b}_2)
    \end{bmatrix} .
\end{align*}

Iterating the above process from $\ell=2$ to $\ell =L$, the following recursive equation can be derived
\begin{align*}
    \tilde{\mathbf{u}}_{\ell} = \tilde{\phi}_{\ell}(\tilde{\mathbf{W}}_{\ell} \tilde{\mathbf{u}}_{\ell-1}+\tilde{\mathbf{b}}_\ell) 
    = \begin{bmatrix}
   \phi_{\ell}(\mathbf{W}_{\ell} \tilde{\mathbf{u}}_{\ell-1}+\mathbf{b}_\ell)\\ \phi_{\ell}(\mathsf{Q}(\mathbf{W}_{\ell}) \tilde{\mathbf{u}}_{\ell-1}+\mathsf{Q}(\mathbf{b}_\ell))
    \end{bmatrix}
\end{align*}
where $\ell = 2,\ldots,L $. Together with (\ref{eq:p_thm1_1}) when $\ell =1$, it yields that 
\begin{align}
    \tilde{\mathbf{u}}_L = \begin{bmatrix}
    \Phi(\mathbf{u})
    \\
    \Phi_\mathsf{Q}(\mathbf{u})
    \end{bmatrix} .
\end{align}

Furthermore, when considering the last layer $\ell = L+1$, the following result can be obtained
\begin{align}
    \tilde{\mathbf{u}}_{L+1} = \mathsf{L}\left(\begin{bmatrix}
    \mathbf{I}_{n_y} & -\mathbf{I}_{n_{y}} 
    \end{bmatrix}\begin{bmatrix}
    \Phi(\mathbf{u})
    \\
    \Phi_\mathsf{Q}(\mathbf{u})
    \end{bmatrix}\right)=  \Phi(\mathbf{u})-  \Phi_\mathsf{Q}(\mathbf{u})
\end{align}
where means $\tilde \Phi(\mathbf{u}) = \Phi(\mathbf{u})-  \Phi_\mathsf{Q}(\mathbf{u})$.

Based on the definition of guaranteed quantization error $\rho(\mathbb{M})$, i.e., Definition \ref{def:error},  we can conclude that
\begin{align}
    \rho(\mathbb{M}) = \sup_{\mathbf{u}\in \mathcal{U}}\left\| \Phi(\mathbf{u})-  \Phi_\mathsf{Q}(\mathbf{u})\right\| =   \sup_{\mathbf{u}\in \mathcal{U}}\left\|\tilde\Phi(\mathbf{u})\right\|  .
\end{align}
The proof is complete.
\end{proof}
\begin{remark} \label{rem2}
Theorem \ref{thm1} implies that we can analyze the merged neural network $\tilde \Phi$ to compute the quantization error between neural network $\Phi$ and its quantized version $\Phi_{\mathsf{Q}}$. This result facilitates the computation process by employing those analyzing tools, such as optimization and reachability analysis tools, for merged neural network $\tilde \Phi$. 
\begin{itemize}
    \item Using the interval arithmetic for neural network, we can employ Moore-Skelboe Algorithm \cite{moore1988inclusion} to search upper bound of $||\tilde{\Phi}(\mathbf{u} )||$ subject to $\mathbf{u} \in \mathcal{U}$ where $\mathcal{U}$ is a compact set. The key to implement  Moore-Skelboe Algorithm is to construct the interval extension of $[\tilde \Phi]:\mathbb{IR}^{n_u} \to \mathbb{IR}^{n_y}$. First, from Theorem 1 in \cite{xiang2020reachable} under the assumption that activation functions are monotonically increasing, the interval extension of  merged neural network  $[\tilde \Phi]$ can be constructed as
    \begin{align}
        [\tilde \Phi] = [\tilde \Phi^{-}, \tilde \Phi^{+}]
    \end{align}
    where $\tilde \Phi^{-}$ and $\tilde \Phi^{+}$ are left (limit inferior) and right (limit superior) bounds of interval $[\tilde \Phi]$ that are defined as follows
    \begin{align*}
    \tilde \Phi^{-}:
        \begin{cases}
            \tilde{\mathbf{u}}_0^{-} = \mathbf{u}^{-} 
            \\
            \tilde{\mathbf{u}}_\ell^{-}=\tilde\phi_{\ell}\left(\begin{bmatrix}
            \tilde{\mathbf{W}}_\ell^{-} & \tilde{\mathbf{W}}_\ell^{+} 
            \end{bmatrix}
            \begin{bmatrix}
            \tilde{\mathbf{u}}_{\ell-1}^{+}
            \\
             \tilde{\mathbf{u}}_{\ell-1}^{-}
            \end{bmatrix}+\tilde{\mathbf{b}}_{\ell}\right)
            \\
            \tilde{\mathbf{y}}^{-} = \tilde{\mathbf{u}}_{L+1}^{-}
        \end{cases}
        \\
            \tilde \Phi^{+}:
        \begin{cases}
            \tilde{\mathbf{u}}_0^{+} = \mathbf{u}^{+} 
            \\
            \mathbf{u}_\ell^{+}=\tilde\phi_{\ell}\left(\begin{bmatrix}
            \tilde{\mathbf{W}}_\ell^{-} & \tilde{\mathbf{W}}_\ell^{+} 
            \end{bmatrix}
            \begin{bmatrix}
            \tilde{\mathbf{u}}_{\ell-1}^{-}
            \\
            \tilde{ \mathbf{u}}_{\ell-1}^{+}
            \end{bmatrix}+\tilde{\mathbf{b}}_{\ell}\right) 
            \\
            \tilde{\mathbf{y}}^{+} = \tilde{\mathbf{u}}_{L+1}^{+}
        \end{cases}
    \end{align*}
in which $\mathcal{U} \subseteq [\mathbf{u}]=[ \mathbf{u}^{-},\mathbf{u}^{+}]$, and 
\begin{align} \label{eq:W_1}
   \mathbf{W}_\ell^{-} &= [\underline{w}_\ell^{i,j}],~\underline{w}_\ell^{i,j} = \begin{cases}
        w_{\ell}^{i,j} & w_{\ell}^{i,j} < 0
        \\
        0 & w_{\ell}^{i,j} \ge 0
    \end{cases}
    \\
    \label{eq:W_2}
    \mathbf{W}_\ell^{+}& = [\overline{w}_\ell^{i,j}],~\overline{w}_\ell^{i,j} = \begin{cases}
        w_{\ell}^{i,j} & w_{\ell}^{i,j} \ge 0
        \\
        0 & w_{\ell}^{i,j} < 0
    \end{cases}
\end{align}
with $w_\ell^{i,j}$, $\underline{w}_\ell^{i,j}$, and $\overline{w}_\ell^{i,j}$ being the elements in $i$-th row and $j$-th column of matrix $\mathbf{W}_\ell$, $\mathbf{W}_\ell^{-}$, and $\mathbf{W}_\ell^{+}$. With the above tractable calculation of $\tilde \Phi^{-}$ and $\tilde \Phi^{+}$, we can perform  Moore-Skelboe Algorithm to compute guaranteed quantization error $\rho(\mathbb{M})$. 
\item Under the framework of reachability analysis of neural networks, the guaranteed quantization error computation problem can be turned into a reachable set computation problem for merged neural network $\tilde\Phi$. Given the input
	set $\mathcal{U}$, the following set
	\begin{align}
	\mathcal{Y} = \left\{\tilde{\mathbf{y}}\in \mathbb{R}^{n_y} \mid \tilde{\mathbf{y}} = \tilde\Phi (\mathbf{u}),~ \mathbf{u} \in \mathcal{U}\right\}
	\end{align}
	is called the output set of neural network (\ref{eq:nn}).  The guaranteed quantization error $\rho(\mathbb{M})$ can be obtained by 
	\begin{align} \label{eq:error}
	    \rho(\mathbb{M}) =  \max\{\mathbf{y} \mid \tilde{\mathbf{y}} \in \mathcal{Y}\} . 
	\end{align}
	The key step is the computation for the reachable set $\mathcal{Y}$. This can be efficiently done through neural network reachability analysis. There exist a number of verification tools for neural networks available for the reachable set computation. The neural network reachability analysis tool can produce the reachable set $\mathcal{Y}$ in the form of a union of polyhedral sets such as NNV \cite{tran2020nnv}, veritex \cite{yang_yamaguchi_hoxha_prokhorov_johnson}, etc. The IGNNV tool computes the reachable set $\mathcal{Y}$ as a union of interval sets \cite{xiang2018output,xiang2020reachable}. With the reachable set ${\mathcal{Y}}$, the guaranteed quantization error $\rho(\mathbb{M})$ can be easily obtained by searching for the maximal value of $\left\|\tilde{\mathbf{y}}\right\| $ in ${\mathcal{Y}}$, e.g., testing throughout a finite number of vertices in the interval or polyhedral sets.
\end{itemize}
\end{remark}

\begin{table}[t!]
\begin{center}
\caption {Neural Network Model Memory Sizes} \label{tab:tab1} 
\begin{tabular}{|c|c|}
\hline
\textbf{Neural Networks}  & \textbf{Size (KB)}  \\
\hline
Original Neural Network ($1 \times 50\times 50\times 50\times 1$)  & 35.0     \\
\hline
Quantized Neural Network ($1 \times 50\times 50\times 50\times 1$)   & 19.6 \\
\hline
\end{tabular}
\end{center}
\end{table} 
\begin{figure}[t!]
\begin{center}
\includegraphics[width=8cm]{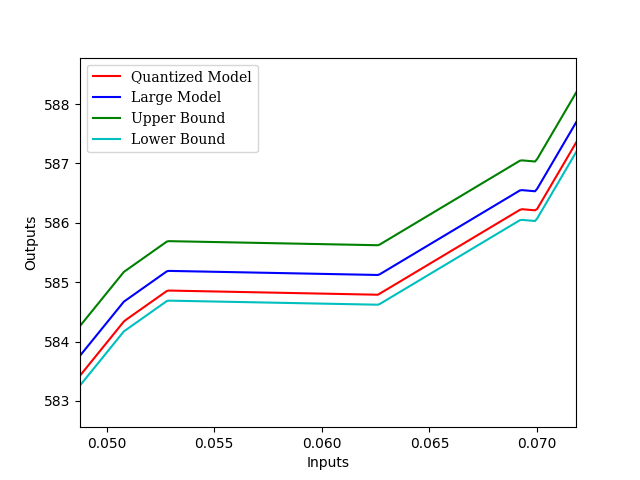}
  \caption{The lower- and upper bounds of outputs of quantized neural networks based on a guaranteed quantization error.}
  \label{fig1}
\end{center}
\end{figure}

\section{Numerical Example}
To verify the effectiveness of the quantization error computation, a numerical example is used. First, a large neural network, $\Phi$, is generated such that it has a 1-D input layer, three hidden layers with 50 neurons in each layer, and a 1-D output layer. Each layer has an activation function of ReLU except for the output layer with a linear function. The weights and biases were randomly initialized, but with the condition that they were normally distributed with a mean of zero and a standard deviation of one. 

After generating $\Phi$, a quantization method was applied to reduce the size of the weights and biases, and to introduce a slight reduction in accuracy. This quantized network is called $\Phi_{\mathsf{Q}}$. While there exist several quantization methods and tools to quantize networks, a basic technique that truncates the weights and biases to 4 decimal places is used in this numerical example.

Next, a merged network $\tilde{\Phi}$ was constructed from $\Phi$ and $\Phi_{\mathsf{Q}}$ according to (\ref{eq:m_nn}). The veritex neural network reachability tool \cite{yang_yamaguchi_hoxha_prokhorov_johnson} computed the reachable output set of $\tilde{\Phi}$ given the input interval normalized as $[0, 1]$. Finally, the quantization error was obtained using (\ref{eq:error}) which is $\rho(\mathbb{M})=0.5008$.
Using this error, the lower and upper bound can be constructed using the following formula: $\Phi(u) \pm \rho(\mathbb{M})$ as shown in Fig. \ref{fig1}. 
Please note that this quantization error is very small compared with the range of outputs. Thus, 
to provide more detail, the figure has been zoomed into an appropriate scale. Moreover, the memory sizes of models are shown in Table \ref{tab:tab1}.

\section{Conclusions}
This paper addressed the guaranteed output error computation problem for neural network compression with quantization. Based on the original neural network and its compressed version resulted from quantization, a merged neural network computation framework is developed, which can utilize optimization-based methods and reachability analysis methods to compute the guaranteed quantization error. At last, numerical examples are proposed to validate the applicability and effectiveness of the proposed approach. Future work will be expanded to include more complex and various neural network architectures such as convolutional neural networks. 

\bibliographystyle{IEEEtran}
\bibliography{reference.bib}

\end{document}